\documentclass{article}

\usepackage[nohyperref,accepted]{icml2021}
\usepackage[ref]{leaf}

\title{When are Iterative Gaussian Processes Reliably Accurate?}

\begin{document}

\twocolumn[
\icmltitle{When are Iterative Gaussian Processes Reliably Accurate?}

\begin{icmlauthorlist}
\icmlauthor{Wesley J. Maddox}{nyu}
\icmlauthor{Sanyam Kapoor}{nyu}
\icmlauthor{Andrew Gordon Wilson}{nyu}
\end{icmlauthorlist}

\icmlaffiliation{nyu}{New York University}

\icmlcorrespondingauthor{Wesley J. Maddox}{wjm363@nyu.edu}

\icmlkeywords{Gaussian processes, conjugate gradients}

\vskip 0.3in
]

\printAffiliationsAndNotice{} 

\begin{abstract}
While recent work on conjugate gradient methods and Lanczos decompositions have
achieved scalable Gaussian process inference with highly accurate point predictions, in several implementations these iterative methods appear to struggle with numerical instabilities in learning kernel hyperparameters, and poor test likelihoods.
By investigating CG tolerance, preconditioner rank, and Lanczos decomposition rank, we provide a particularly simple prescription to correct these issues: we recommend that one should use a small CG tolerance ($\epsilon \leq 0.01$) and a large root decomposition size ($r \geq 5000$). Moreover, we show that L-BFGS-B is a compelling optimizer for Iterative GPs, achieving convergence with fewer gradient updates.
\end{abstract}

\section{Introduction}

There are now many methods for scaling Gaussian processes (GPs), including finite basis expansions \citep{rahimi2007random,solin2020hilbert}, inducing point methods \citep{snelson2006sparse,titsias2009variational,hensman2013gaussian}, and iterative numerical methods \citep{gibbs1996cient,cutajar2016preconditioning,gardner_gpytorch:_2018}. These methods all work in various ways to scale linear solves and log determinant computations involving covariance matrices.
In some type of limit, these methods all generally converge to the ``exact" Gaussian process regressor, but at a computational expense.

Iterative methods, such as the preconditioned conjugate gradient (CG) and Lanczos algorithms that form the backbone of the popular software library GPyTorch \citep{gardner_gpytorch:_2018}, are rapidly growing in popularity, in part due to their accuracy and effective use of GPU parallelization. These methods are different from other scalable GP approximations in that the accuracy of their solves can be precisely controlled by hyperparameters such as tolerances. Thus these methods are considered ``exact'' in the same sense as mathematical optimization --- accurate to arbitrary precision.
In practice, these methods can be more precise than the ``exact'' Cholesky-based alternatives, due to round-off errors in floating point numbers \citep{gardner_gpytorch:_2018}.

However, these iterative methods have demonstrated mysterious empirical behaviour. For example, in \citet{jankowiak2020parametric} the ``exact'' iterative GPs generally have impressive test RMSE, but surprisingly poor test likelihood. Similar results can be found for the CG-based SKI method in \citet{sun2021scalable}. 
Moreover, we find there are sometimes numerical instabilities in training kernel hyperparameters with iterative methods, for instance leading to increasingly large kernel length-scales and increasingly poor test performance, as marginal likelihood optimization proceeds. Recent work has explicitly noted issues with iterative CG-based methods \citep{artemev_tighter_2021, potapczynski_bias-free_2021}. For example, \citet{potapczynski_bias-free_2021} show that CG methods can provide biased estimates, and propose random truncations as a remedy. Similarly, \citet{wenger2021reducing} propose variance reduction techniques to reduce the bias of CG and Lanczos methods.

In this work, we study the empirical convergence of iterative Gaussian process models based on conjugate gradients and Lanczos decompositions.
Our empirical study evaluates the hyperparameters of these methods, namely the size of the test time Lanczos decomposition $r,$ the tolerance of conjugate gradients $\tau$, and the rank of the pre-conditioner used for conjugate gradients solutions, $w.$ The primary source of the mysterious empirical issues for iterative methods are overly relaxed defaults in GPyTorch. 

We provide simple prescriptions for adjusting these defaults in order for practitioners to achieve reliably good performance with iterative methods. 
We also highlight the practical benefits of L-BFGS-B for marginal likelihood optimization, which is not presently a standard method for learning kernel hyperparameters.

In Figure \ref{fig:explainer}, we give a graphical explanation of our results on one output dimension of the sarcos dataset $n = 44484, d=21$. In the first two panels from left, we consider the optimization trajectories (using Adam) of both high tolerance iterative GPs (with high values of CG tolerance and a small preconditioner) against low tolerance iterative GPs (with our recommendations of a lower CG tolerance value and a larger preconditioner size), finding that the optimization paths tend to be more stable for low tolerance values.
Then, in the center left panel, we compare test set NLL as a function of the root decomposition size, finding that significant improvements in NLL can be made by simply increasing the size of the root decomposition. 
Finally, in the right panel, we consider test set MSE as a function of CG tolerance, finding that the CG tolerance also plays an important role in the accuracy of the method.

Code for all experiments can be found  \href{https://github.com/wjmaddox/benchmarking_iterative_gps}{\underline{here}}.

\begin{figure*}[t!]
    \centering
    \includegraphics[width=\textwidth]{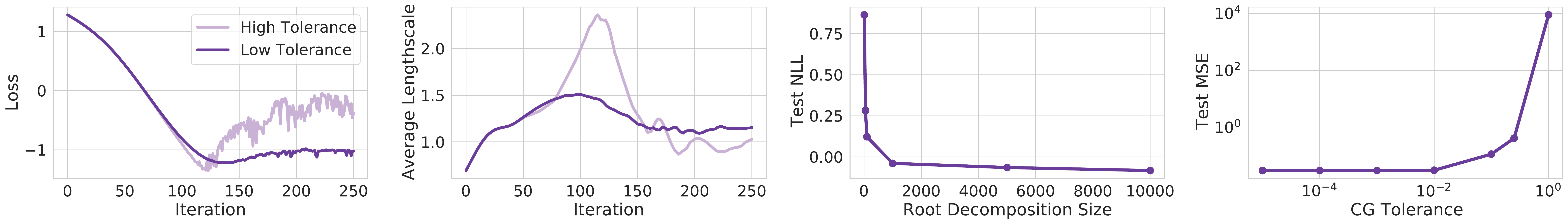}
    \caption{\textbf{Far Left:} Loss as a function of iteration for both high tolerance (GPyTorch defaults) and low tolerance (CG tolerance of $10^{-4}$ and a larger preconditioner size); low tolerance produces more stable optimization trajectories. \textbf{Center Left:} Outputscale as a function of iteration for both high tolerance and low tolerance CG; again, low tolerance provides more stable trajectories. \textbf{Center right:} Test set negative log likelihood as a function of root decomposition size; a larger root decomposition tends to produce lower NLL (which is better) but this can begin increasing at larger sizes. \textbf{Far right:} Test set MSE as a function of test set CG tolerance; a lower tolerance is more accurate.}
    \label{fig:explainer}
\end{figure*}

\section{Gaussian Process Regression}
Gaussian processes (GPs) are a non-parametric method that model priors directly over functions, and provide well-calibrated uncertainties \citep{rasmussen_gaussian_2008}. In this work, for simplicity, we focus on the regression setting with isotropic noise. For a dataset $\dset$ of size $n$, the relation between $d$-dimensional inputs $\mbf{X} \in \reals^{n\times d}$, and corresponding outputs $\mbf{y} \in \reals^{n}$, is modeled using a GP prior ${f(\mbf{X}) \sim \gp{\mu(\mbf{X}), k(\mbf{X},\mbf{X})}}$ and a Gaussian observation likelihood ${\mbf{y}} \sim \gaussian{f(\mbf{X}), \sigma^2\mbf{I}}$. The prior is fully specified by a mean function $\mu(\cdot)$, and a kernel function $k(\cdot,\cdot)$, both with parameters collectively denoted as $\theta$. We take the mean function to be zero, as typical in literature. The kernel function induces a covariance matrix between two sets $A \in \reals^{m\times d}$ and $B \in \reals^{m^\prime\times d}$, denoted as $K_{A,B} \in \reals^{m\times m^\prime}$. Therefore, the prior covariance matrix is denoted by $K_{\mbf{X},\mbf{X}} \in \reals^{n\times n}$. 

Gaussian process inference aims to find the posterior over the functions $f$ for $n_\star$ novel inputs $\mbf{X}_\star$, which is also a Gaussian and available in closed-form as,
\begin{align}
p(f(\mbf{X}_\star) \mid \mbf{X}_\star,& \dset, \theta) = \mathcal{N}(\mu(\mbf{X}_\star), \Sigma(\mbf{X}_\star)), \label{eq:gp_post} \\
\mu(\mbf{X}_\star) &= K_{\mbf{X}_\star, \mbf{X}} \widehat{K}_{\mbf{X},\mbf{X}}^{-1} \mbf{y}~, \nonumber \\
\Sigma(\mbf{X}_\star) &= K_{\mbf{X}_\star, \mbf{X}_\star} - K_{\mbf{X}_\star, \mbf{X}} \widehat{K}_{\mbf{X},\mbf{X}}^{-1} K_{\mbf{X}_\star, \mbf{X}}^\top ~, \nonumber
\end{align}
where $\widehat{K}_{\mbf{X},\mbf{X}} = K_{\mbf{X},\mbf{X}} + \sigma^2 \mbf{I}$.
The parameters $\theta$ are chosen by maximizing the marginal $\log$-likelihood (MLL) $L_\theta$,
\begin{align}
\underbrace{\log p(\mbf{y} \mid \mbf{X}, \theta)}_{L_\theta} \propto ~& -\frac{1}{2}\mbf{y}^\top\widehat{K}_{\mbf{X},\mbf{X}}^{-1}\mbf{y} - \frac{1}{2} \logdet{\widehat{K}_{XX}}~, \label{eq:gp_mll}
\end{align}
and its gradients are given by,
\begin{small}
\begin{align}
\hspace*{-1em}\pd{L_\theta}{\theta} = \mbf{y}^\top \widehat{K}_{\mbf{X},\mbf{X}}^{-1} \pd{\widehat{K}_{\mbf{X},\mbf{X}}}{\theta}\widehat{K}_{\mbf{X},\mbf{X}}^{-1} \mbf{y} + \tr{\widehat{K}_{\mbf{X},\mbf{X}}^{-1} \pd{\widehat{K}_{\mbf{X},\mbf{X}}}{\theta}}. \label{eq:gp_mll_grad}
\end{align}
\end{small}

Traditional methods to compute \cref{eq:gp_post,eq:gp_mll,eq:gp_mll_grad} use Cholesky factorization to solve the linear systems involving $\widehat{K}_{\mbf{X},\mbf{X}}^{-1}$, and the determinant computations involving $\widehat{K}_{\mbf{X},\mbf{X}}$ \citep{rasmussen_gaussian_2008}. This incurs an expensive cost of \bigo{n^3}, making exact GPs feasible only for less than $n=10,000$ data points and severely limiting scalability.

\subsection{Iterative Gaussian Processes}
\label{sec:igp}

\citet{gardner_gpytorch:_2018} propose conjugate gradients (CG) to solve $\widehat{K}_{\mbf{X},\mbf{X}}^{-1} \mbf{y}$, and stochastic trace estimator \citep{Hutchinson1989ASE} to compute the derivative of $log$-determinant $\tr{\widehat{K}_{\mbf{X},\mbf{X}}^{-1} \pd{\widehat{K}_{\mbf{X},\mbf{X}}}{\theta}}$ in \cref{eq:gp_mll_grad} as, 
\begin{small}
\begin{align*}
\mathbb{E}_{p(z)}\left[ z^\top \widehat{K}_{\mbf{X},\mbf{X}}^{-1} \pd{\widehat{K}_{\mbf{X},\mbf{X}}}{\theta} z \right]
\approx \frac{1}{B} \sum_{b=1}^B (z_b^\top \widehat K_{\mbf{X},\mbf{X}}^{-1} \pd{\widehat{K}_{\mbf{X},\mbf{X}}}{\theta} z_b)
\end{align*}
\end{small}
where we use $B$ probe vectors $\{z_b\}_{b=1}^B$. Notably, we can again use conjugate gradients to solve $\widehat{K}_{\mbf{X},\mbf{X}}^{-1} z_b$. 

Consequently, for kernel matrices $K_{\mbf{X},\mbf{X}}$, using $r \ll n$ iterations of conjugate gradients produces exact inference up to machine precision in \bigo{rn^2} time, achieving significant computational gains. The MLL can be computed using $B+1$ CG solves and is easily parallelized. Exact GPs are recovered when $r = n$. A flurry of recent work \citep{wang_exact_2019,potapczynski_bias-free_2021,artemev_tighter_2021,kapoor2021skiing} has shown the effectiveness of these methods for highly scalable Gaussian processes.

Further, \citet{pleiss_constant-time_2018} propose to store the single solve of $\textcolor{blue}{\mbf{m}} = \widehat{K}_{\mbf{X},\mbf{X}}^{-1}\mbf{y}$ at test time as the predictive mean cache, while similarly computing a rank $k$ Lanczos decomposition of $\widehat{K}_{\mbf{X},\mbf{X}}^{-1} \approx \textcolor{blue}{RR^\top}$ (the predictive covariance cache) and storing that for all new test points. The predictive mean and variance for a test point $\mbf{x}_\star$ are then given by,
\begin{align}
\begin{split}
\mu(\mbf{x}_\star) &= K_{\mbf{x}_\star,\mbf{X}}\textcolor{blue}{\mbf{m}}~, \\
\Sigma(\mbf{x}_\star) &= K_{\mbf{x}_\star, \mbf{x}_\star} - K_{\mbf{x}_\star,\mbf{X}}\textcolor{blue}{RR^\top} K_{\mbf{x}_\star,\mbf{X}}^\top~.    
\end{split}
\end{align}
Computing the Lanczos decomposition costs $\mathcal{O}(kn^2),$ so that computing the predictive mean and covariance after this pre-computation cost only $\mathcal{O}(knn_\star)$. In general, Lanczos will tend to find the largest and smallest eigenvalues of $\widehat{K}_{\mbf{X},\mbf{X}}^{-1}$, converging from the \emph{outside in} so to speak \citep[Ch. 7]{demmel1997applied}. Collectively, we call inference involving all the above numerical methods as \emph{Iterative Gaussian Processes} \citep{potapczynski_bias-free_2021}.

\section{Understanding Iterative GP Approximations}
\label{sec:understanding_igp}

Understandably, the \emph{speed} and \emph{quality} of Iterative GPs are crucially reliant on conjugate gradients. (i) First, the number of CG iterations $r$ depend on the condition number of $\widehat{K}_{\mbf{X},\mbf{X}}$, which grow somewhat with $n$. To alleviate this concern, \citet{gardner_gpytorch:_2018} use a pivoted Cholesky pre-conditioner of rank $w$, noting that low ranks are sufficient. Alternatively, \citet{artemev_tighter_2021} propose an inducing point-based pre-conditioner. (ii) Second, in practice, we often use a pre-determined error tolerance $\epsilon$ to truncate CG iterations. A higher tolerance implies that the solve may use fewer iterations than $r$, while a lower tolerance may require more iterations than $r$.  A high tolerance, however, has been noted to be detrimental to MLL \cref{eq:gp_mll} maximization \citep{artemev_tighter_2021}, and may lead to noisy training curves affecting final performance \citep{kapoor2021skiing}. (iii) Finally, CG truncations at a value $r \ll n$ introduces approximation bias, which is also detrimental to the maximization of MLL \citep{potapczynski_bias-free_2021}.

Owing to the considerations above, and their emphasis in prior work \citep{gardner_gpytorch:_2018,pleiss_constant-time_2018,pleiss_fast_2020}, we focus our study on analyzing \emph{three} key hyperparameters \textemdash~ (i) the CG tolerance $\epsilon$, (ii) rank of the pivoted Cholesky pre-conditioner $w$, and (iii) the rank of the Lanczos decomposition $k$ at test time. In addition, we analyze the interaction of these hyperparameters to both a first-order optimizer Adam \citep{kingma2014adam}, and a second-order optimizer L-BFGS-B \citep{zhu1997algorithm}.

Historically, (nearly) second order optimizers such as L-BFGS-B have been preferred for optimizing Gaussian process hyper-parameters \citep{rasmussen_gaussian_2008}; however, they tend to perform poorly in the presence of noisy estimates of gradients \citep{gardner_gpytorch:_2018,balandat_botorch_2020}.
Thus, \citet{gardner_gpytorch:_2018,wang_exact_2019} tend to fit their GP models with first order optimizers such as Adam, as L-BFGS-B on all but small datasets tends to be infeasible.
One of our goals in our study is to determine the hyper-parameters under which we can use L-BFGS-B to train hyper-parameters of Gaussian process models while still using iterative methods.

\subsection{Lanczos Decomposition Rank and Posterior Likelihood}

\begin{restatable}[]{proposition}{lanczosposterior}
\label{prop:lanczos_posterior}
For a single test data point $\mbf{x}_\star$, increasing the rank $k$ of an approximate eigen-decomposition will decrease the posterior variance $\sigma^2$.
\end{restatable}

The proof of \cref{prop:lanczos_posterior} is in \cref{sec:tech_app}, and proceeds algebraically.
\cref{prop:lanczos_posterior} implies that increasing the rank of a Lanczos decomposition at test reduces the posterior variance of the GP, making the GP more confident about its predictions.
As we show next, this effect can have counterintuitive effects on the test time negative log likelihood (NLL).
As a function of the variance $\sigma^2,$ the NLL is,
\begin{align}
\mathrm{NLL}(\sigma^2) := \frac{1}{2}\log \sigma^2 + \frac{1}{2\sigma^2}(\mu - y)^2~, \label{eq:gp_nll}
\end{align}
where $\mu$ is the predictive mean.
Differentiating and setting the derivative equal to zero finds that $\sigma^2 = (\mu - y)^2$ is a minimum (the global minimum) agrees with standard results where the maximum likelihood estimator of $\sigma^2$ in regression is the sum of squared errors of the predictor.

\begin{figure}[!ht]
\centering
\includegraphics[width=.55\linewidth]{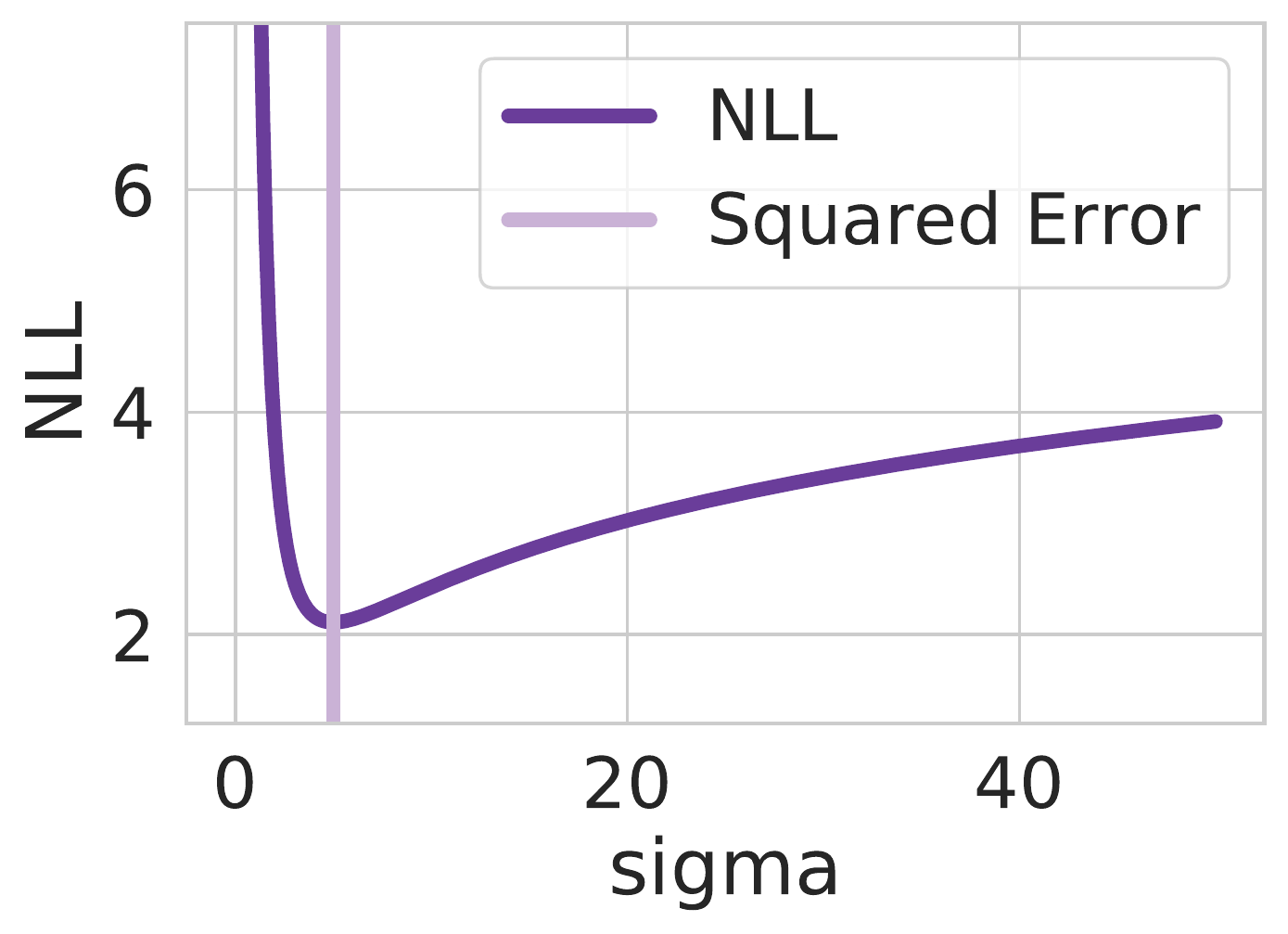}
\vspace{-1em}
\caption{NLL \cref{eq:gp_nll} as a function of the posterior standard deviation $\sigma$. The NLL is minimized when $\sigma^2 = (y - \mu)^2$.}
\label{fig:nll_schematic}
\end{figure}

We show in \cref{fig:nll_schematic} that the NLL is minimized when $\sigma^2$ is set as the squared error of the prediction.
\cref{fig:nll_schematic} also demonstrates that if we control the posterior variance (the square of $\sigma$), then the NLL is \emph{not monotonic} as we decrease (or increase) the posterior variance. 
Thus, as we increase the rank of the Lanczos decomposition, we decrease the variance --- beginning by decreasing the NLL towards the NLL's minimizer (which may or may not correspond to the full rank Cholesky solution).
However, once we have reached the NLL's minimizer, continuing to decrease the posterior variance then increases the NLL again.

\begin{table*}[t]
\renewcommand{\arraystretch}{1.2}
\caption{
Recommended settings for stable training (and testing) with iterative GPs, especially when using L-BFGS-B.
}
\label{table:hyper_recommendation}
\begin{sc}
\begin{adjustbox}{width=\linewidth}
\begin{tabular}{c|c|c|c|c|c}
\toprule
\multicolumn{2}{c|}{Training (Kernel Learning)} & \multicolumn{2}{c|}{Predictive Mean} & \multicolumn{2}{c}{Predictive Distributions} \\
\midrule
Hyperparameter & Value(s) & Hyperparameter & Value(s) & Hyperparameter & Value(s) \\
\midrule
\textbf{CG Tolerance} $\epsilon$ & $10^{-3}$ &\textbf{CG Tolerance} $\epsilon$ & $10^{-2}$ & \textbf{CG Tolerance} $\epsilon$ & $10^{-2}$ \\
\textbf{CG Preconditioner Rank} $w$ & $50$ & \textbf{CG Preconditioner Rank} $w$ & $50$ & \textbf{CG Preconditioner Rank} $w$ & $50$ \\
\textbf{Lanczos Decomposition Rank} $k$ & - & \textbf{Lanczos Decomposition Rank} $k$ & - & \textbf{Lanczos Decomposition Rank} $k$ & $\geq5000$ \\
\bottomrule
\end{tabular}
\end{adjustbox}
\end{sc}
\end{table*}

These two settings imply two different behaviors \textemdash~ if the NLL decreases as we increase the root decomposition, then the hyper-parameters of the GP model are under-fit as the squared residual error on the test set is less than the estimated posterior variance.
If the opposite occurs (NLL increases as we increase the root decomposition), then the GP model itself is over-fit, the squared residual error on the test set is greater than the estimated posterior variance, e.g. we are over-confident about our predictions.

\section{Results}

Through our experiments, we establish that (i) increasing the Lanczos decomposition rank $r$ fixes the discrepancy between exact GPs and Iterative GPs; (ii) decreasing the CG tolerance allows us to run L-BFGS-B\footnote{We use the full-batch implementation available at \url{https://github.com/hjmshi/PyTorch-LBFGS}.} \citep{zhu1997algorithm}, which has not been investigated in prior literature.

\textbf{Experimental Setup}: We consider four benchmark datasets from the UCI repository \citep{Dua2019}, elevators, protein, bike and poletele, standardized to zero sample mean and unit sample variance using the training set. We run the Cholesky and Iterative GPs over five random splits of the data, reporting the mean. As in \citet{gardner_gpytorch:_2018}, we use the Matérn-$5/2$ kernel with automatic relevance determination (ARD) and constant means. For optimization, we use a learning rate of $0.05$ for $2000$ steps or until convergence. As established in \cref{sec:understanding_igp}, the key hyperparameters of importance are the CG error tolerance $\epsilon$, rank of pivoted Cholesky pre-conditioner $w$, and the rank of the Lanczos decomposition $k$ at test time. In addition, we want to understand the interaction of these parameters with optimizers Adam and L-BFGS-B. As in \citet{wang_exact_2019}, which used L-BFGS-B only for pre-training initialization, we use the default setting of $10$ memory vectors.

\textbf{Effect of Tolerance and Pre-conditioner Size}:
We find that the the discrepancy between the NLL achieved by Cholesky-based inference and Iterative GPs can be attributed primarily to (i) too large of a CG tolerance, and (ii) too small of a test time root decomposition.

Varying the (test-time) root decomposition size on elevators for a large preconditioner, and use a smaller CG tolerance ($0.01$ or less), as in \cref{fig:nll_precond_elevators}, we are able to reach the Cholesky baseline in terms of NLL for a root decomposition size of about $10,000$.
By comparison, the RMSE converges much faster to the Cholesky baseline (shown in grey again) for preconditioner sizes of $50$ and $100$ on the same dataset (\cref{fig:rmse_elevators} left three panels).
In addition, we vary the CG tolerance in both double and float precision for three different pre-conditioner sizes, finding that the RMSE converges very quickly, and for CG tolerances less than $0.1$, the test time RMSE differences are negligible.

While the previous experiments compared a mix of first order and second order methods as in \citet{wang_exact_2019}, in \cref{fig:nll_lbfgs}, we show CG tolerance is also important for the success of L-BFGS-B only.
In \cref{fig:rmse_elevators} far right, we use a tolerance of $1.0$ at train time, varying the test time CG tolerance, finding that the performance of Adam and L-BFGS-B decay for high tolerances, with the Adam performance decaying more.
Finally, in 
\cref{fig:nll_lbfgs}, we vary preconditioner size and CG tolerance for fixed root decomposition, finding again that low test-time CG tolerance is imperative for good results with L-BFGS-B.

\section{Recommendations}

As shown in Table \ref{table:hyper_recommendation}, we recommend a low CG tolerance for training (of around $10^{-3}$), along with a larger preconditioner size (of $50$), for more reliable performance. Between these, a lower CG tolerance should be preferred.
These recommendations apply to both training with the GP marginal log likelihood as well as computing the predictive mean.
For predictive distributions, we also recommend computing a larger Lanczos decomposition rank (of around $5000$).

As of this writing, default GPyTorch settings\footnote{\url{https://docs.gpytorch.ai/en/stable/settings.html}} are a tolerance of $1.0$ at train time, and $10^{-3}$ at test time, and a root decomposition size of $100$.
We hypothesize that test NLL is not the only quantity that can be significantly hindered as a result of using too small a root decomposition, as sampling depends on similar quantities. In the future, we hope to expand our benchmarking to use Iterative GPs for large-scale data in Bayesian optimization \citep{balandat_botorch_2020}.

\clearpage

\begin{figure*}[!ht]
    \centering 
    \includegraphics[width=.7\linewidth]{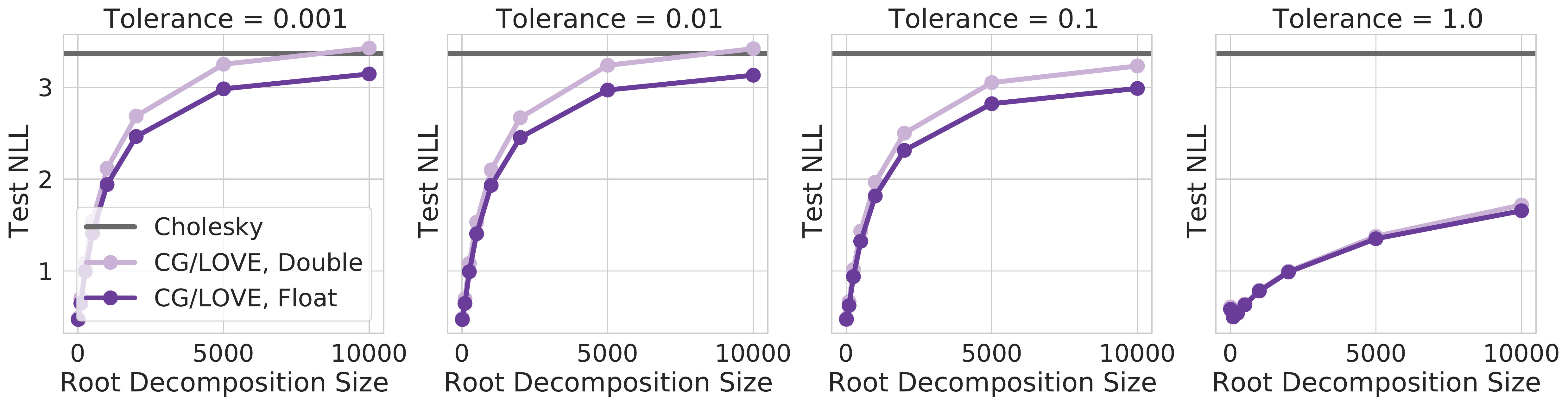}
    \vspace{-1em}
    \caption{Test set NLL as a function of root decomposition size and CG tolerance on elevators. Decreasing the CG tolerance with a preconditioner of size $100$ produces more accurate solves and thus more accurate NLL estimation for large root decomposition sizes.}
    \label{fig:nll_precond_elevators}
\end{figure*}
\begin{figure*}[!ht]
    \centering
\begin{tabular}{cc}
    \includegraphics[width=.6\linewidth]{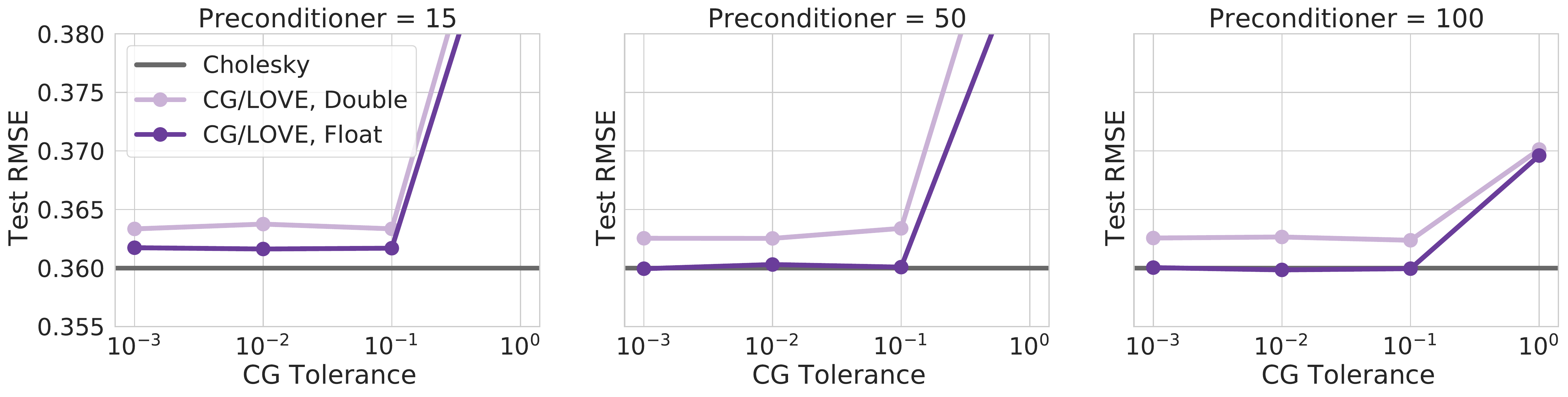} &  \includegraphics[width=.2\linewidth]{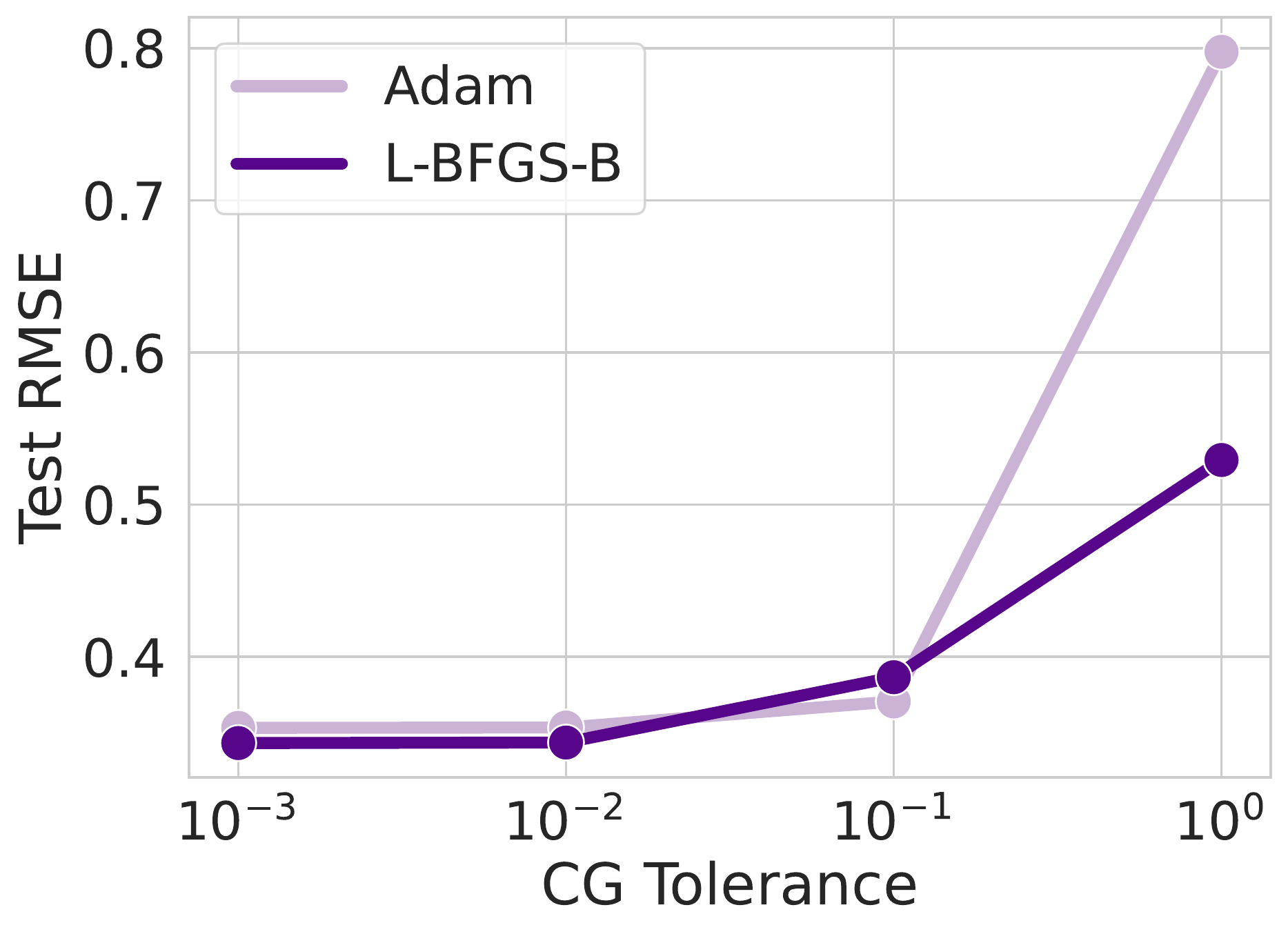}
\end{tabular}
\vspace{-1em}
    \caption{RMSEs as a function of preconditioner size on elevators. A small enough CG tolerance ($0.1$ or less) produces very similar results to the cholesky baseline even for small preconditioner sizes.}
    \label{fig:rmse_elevators}
\end{figure*}

\begin{figure*}[!ht]
    \centering
    \includegraphics[width=.8\linewidth]{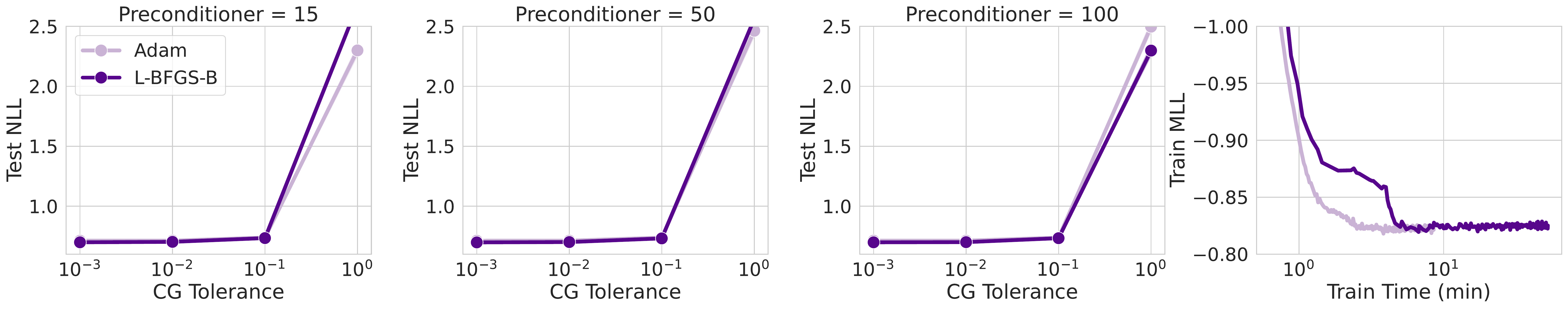}
    \vspace{-1em}
    \caption{\textbf{Left Three Panels: } Test set NLL on protein for both Adam (with L-BFGS-B pre-training) and L-BFGS-B as a function of preconditioner size and CG tolerance. CG tolerance strongly matters for performance, even more so for L-BFGS-B. \textbf{Far Right Panel: } Training time on a single GPU; Adam is approximately twice as fast to reach the same training MLL (and ultimately test set NLL).}
    \label{fig:nll_lbfgs}
\end{figure*}

\clearpage
\section*{Acknowledgements}

We thank Sam Stanton and Alex Wang for helpful discussions.

\bibliographystyle{icml2021}
\bibliography{references}

\clearpage
\appendix

\twocolumn[
\icmltitle{Appendix for \\ When is an Iterative Gaussian Process Numerically Exact?}

\begin{icmlauthorlist}
\icmlauthor{Wesley J. Maddox}{nyu}
\icmlauthor{Sanyam Kapoor}{nyu}
\icmlauthor{Andrew Gordon Wilson}{nyu}
\end{icmlauthorlist}

\vskip 0.3in
]

\section{Technical Appendix}
\label{sec:tech_app}

We prove \cref{prop:lanczos_posterior} here.

\lanczosposterior*
\begin{proof}
For a given test data point $\mbf{x}_\star$, we denote the posterior variance using rank $k$ as $\sigma_k^2(\mbf{x}_\star)$. We also note a full eigendecomposition ${\widehat{K}_{\mbf{X},\mbf{X}} =  Q\widehat{\Lambda}Q^\top}$, where ${\widehat{\Lambda} = \Lambda + \sigma^2\mbf{I}}$ such that the diagonal of matrix $\Lambda$ contains the eigenvalues of the covariance matrix $K_{\mbf{X},\mbf{X}}$. Then, we aim to show,
\begin{align}
\sigma_k^2(\mbf{x}_\star) - \sigma_{k+1}^2(\mbf{x}_\star) \geq 0~.
\end{align}
Denoting by $Q_{:k}$ and $\widehat{\Lambda}_{:k}$ a slice of the first $k$ eigenvectors and eigenvalues respectively, we expand the posterior variances as in \cref{eq:gp_post}. Therefore, we need,
\begin{align}
\begin{split}
\left(K_{\mbf{x}_\star,\mbf{x}_\star} - K_{\mbf{x}_\star,\mbf{X}}Q_{:k}\widehat{\Lambda}_{:k}^{-1}Q_{:k}^\top K_{\mbf{x}_\star,\mbf{X}}^\top\right) -&~ \\ 
\left(K_{\mbf{x}_\star,\mbf{x}_\star} - K_{\mbf{x}_\star, \mbf{X}}Q_{:k+1}\widehat{\Lambda}_{:k+1}^{-1}Q_{:k+1}^\top K_{\mbf{x}_\star,\mbf{X}}^\top \right) &\geq 0~.
\end{split}
\end{align}
By appending zeros to the rank-$k$ diagonal matrix $\widehat{\Lambda}_{:k}^{-1}$, and factorizing, we have,
\begin{align}
K_{\mbf{x}_\star, \mbf{X}}Q_{:k+1}\left( \widehat{\Lambda}_{:k+1}^{-1} - \left[\widehat{\Lambda}_{:k}^{-1}; \mbf{0} \right] \right) Q_{:k+1}^\top K_{\mbf{x}_\star, \mbf{X}}^\top \geq 0 \\
(1 / \widehat{\lambda}_{k+1}) \left(K_{\mbf{x}_\star, \mbf{X}}Q_{:k+1}\right)^\top\left(K_{\mbf{x}_\star, \mbf{X}}Q_{:k+1}\right) \geq 0~, \label{eq:prop1_proof}
\end{align}
where $\widehat{\lambda}_{k} > 0$ denotes the $k^{\mathrm{th}}$ eigenvalue, and \cref{eq:prop1_proof} holds by virtue of being an inner product. Hence, the posterior variance reduces as we increase $r$ at test time.
\end{proof}
\section{Hyperparameters Matrix}

We list all the hyperparameters studied in \cref{table:hypers}.
For model fitting, we use Matern-$5/2$ kernels with constant means and automatic relevance determination\footnote{\url{https://botorch.org/api/models.html\#module-botorch.models.gp\_regression} except that we place a softplus transform on the raw noise.} and use an exponential moving average stopping rule to monitor convergence\footnote{From \url{https://botorch.org/api/optim.html\#botorch.optim.fit.fit\_gpytorch\_torch}.}

\begin{table}[!ht]
\renewcommand{\arraystretch}{1.2}
\caption{
We document all the settings and hyperparameters involved.
}
\label{table:hypers}
\vskip 0.15in
\begin{sc}
\begin{adjustbox}{width=\linewidth}
\begin{tabular}{l|c}
\toprule
Hyperparameter & Value(s) \\
\midrule
Kernel Family & $\mathrm{Mat\acute{e}rn}$-5/2 \\
Max. Epochs & $2000$  \\
Optimizer & \{$\mathrm{Adam}$, L-BFGS-B \} \\
L-BFGS-B Memory Vectors & $10$ \\
Learning Rate & $0.05$ \\
Max. CG Iterations & $500$ \\
\textbf{CG Tolerance} $\epsilon$ & $\{10^{-3}, 10^{-2}, 10^{-1}, 1.\}$ \\
\textbf{CG Preconditioner Rank} $w$ & \{$15$, $50$, $100$\} \\
\textbf{Lanczos Decomposition Rank} $k$ & \parbox{.6\linewidth}{\{$15$, $100$, $500$, $1000$, $2500$, $5000$, $10000$\}} \\
\bottomrule
\end{tabular}
\end{adjustbox}
\end{sc}
\end{table}

\section{Further Experimental Results}
All experimental code can be found at \url{https://github.com/wjmaddox/benchmarking_iterative_gps}.

\begin{figure*}[!ht]
    \centering
    \includegraphics[width=.8\linewidth]{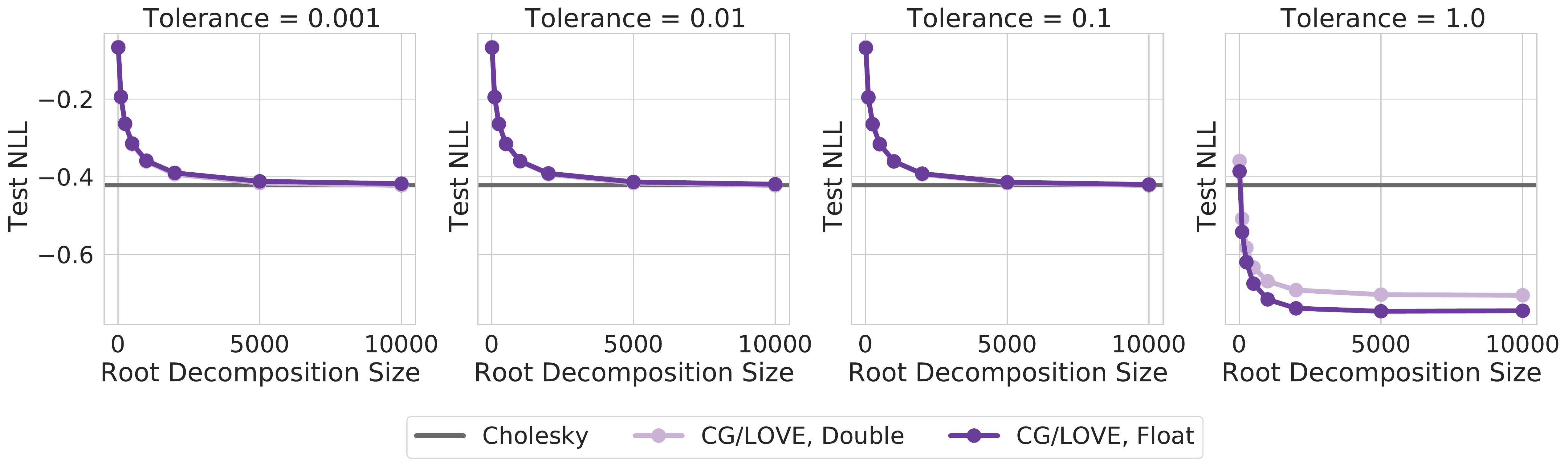}
    \caption{Test NLL as a function of root decomposition size and CG tolerance on \texttt{bike} dataset. Decreasing the CG tolerance with a preconditioner of size $15$ produces more accurate solves and thus more accurate NLL estimation for large root decomposition sizes.}
    \label{fig:nll_precond_bike}
\end{figure*}

\begin{figure*}[!ht]
    \centering
    \includegraphics[width=.8\linewidth]{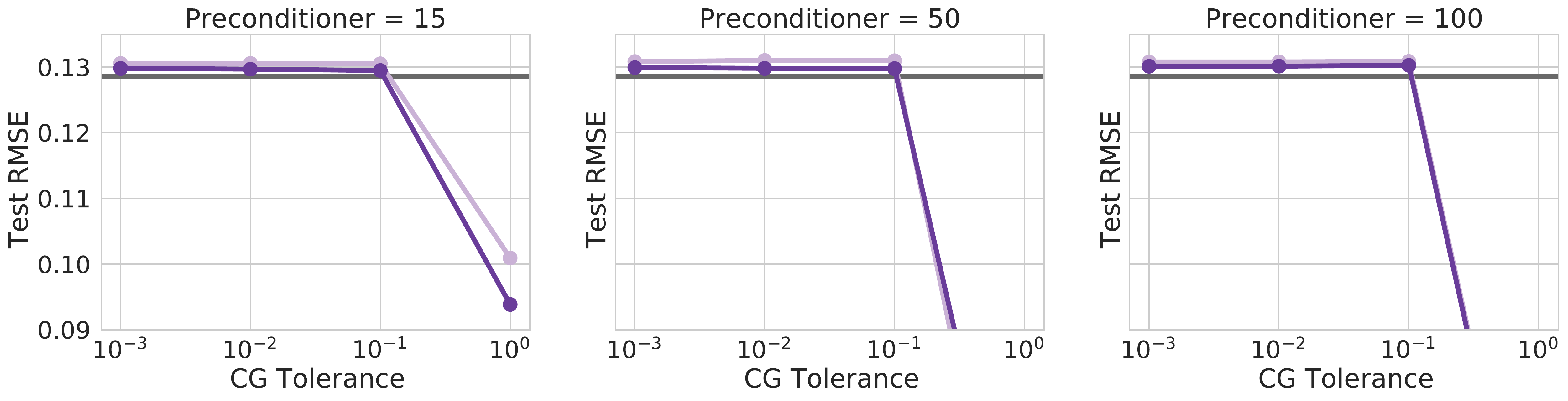}
    \caption{When plotting the test RMSE on \texttt{bike} dataset against the CG error tolerance, we find that the CG preconditioner rank does not significantly affect the performance.}
    \label{fig:rmse_bike}
\end{figure*}

\paragraph{Default settings of GPyTorch}
\begin{figure*}[!ht]
    \centering
    \includegraphics[width=.8\linewidth]{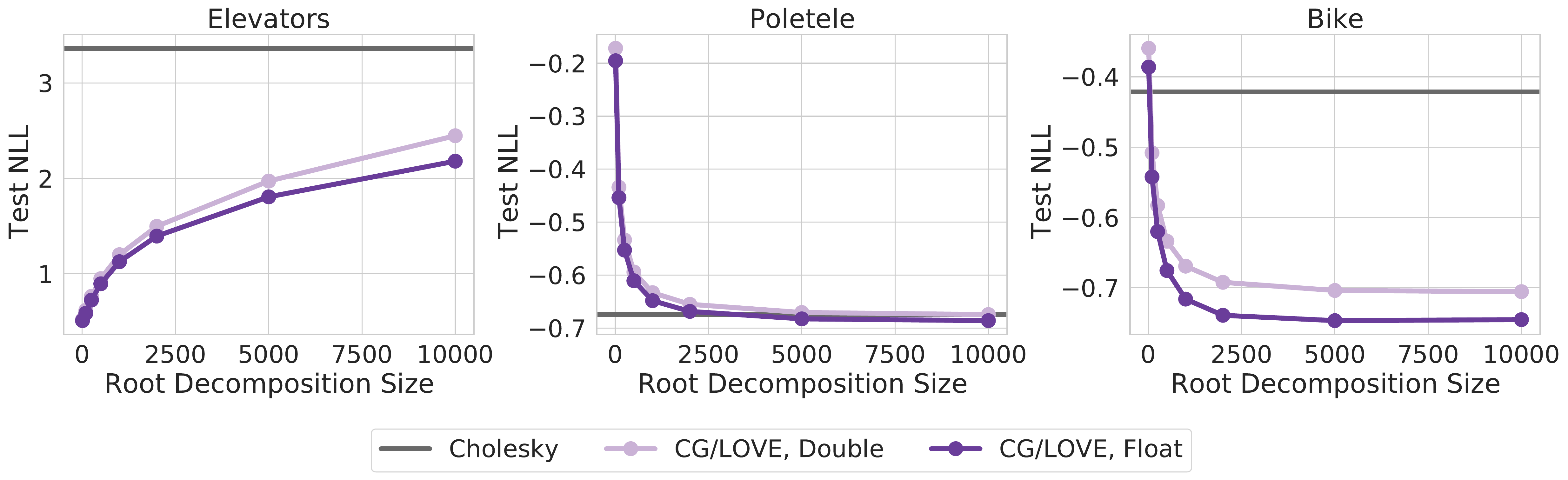}
    \caption{Test NLL (standardized) on \texttt{elevators}, \texttt{bike}, and \texttt{poletele} as a function of the root decomposition size for LOVE \citep{pleiss_constant-time_2018} using GPyTorch default settings for preconditioners and CG tolerance. Only on \texttt{poletele} does any size root decomposition reach the same NLL as the Cholesky implementation.}
    \label{fig:nll_defaults}
\end{figure*}

In \cref{fig:nll_defaults}, we display the standardized negative log likelihood on elevators, bike and poletele as a function of the root decomposition size for LOVE for the gpytorch default settings (CG tolerance of $1$ and a rank $15$ preconditioner).
On elevators and bike, there are significant differences between the NLLs, even in double precision and for very large root decompositions. 
Furthermore, elevators shows evidence of over-fitting as the NLL increases when the root decomposition grows while bike (and poletele) show evidence of some under-fitting as the NLL decreases when the root decomposition is larger.

By comparison, the actual accuracy, solely determined by the accuracy of the CG solve $\widehat{K}_{\mbf{X},\mbf{X}}^{-1} \mbf{y}$ is much closer for the default settings. 
We show these results in \cref{fig:rmse_defaults}, where intriguingly a less accurate solve can sometimes reduce the RMSE, perhaps due to increased regularization.

\begin{figure*}[!ht]
    \centering
    \includegraphics[width=.8\linewidth]{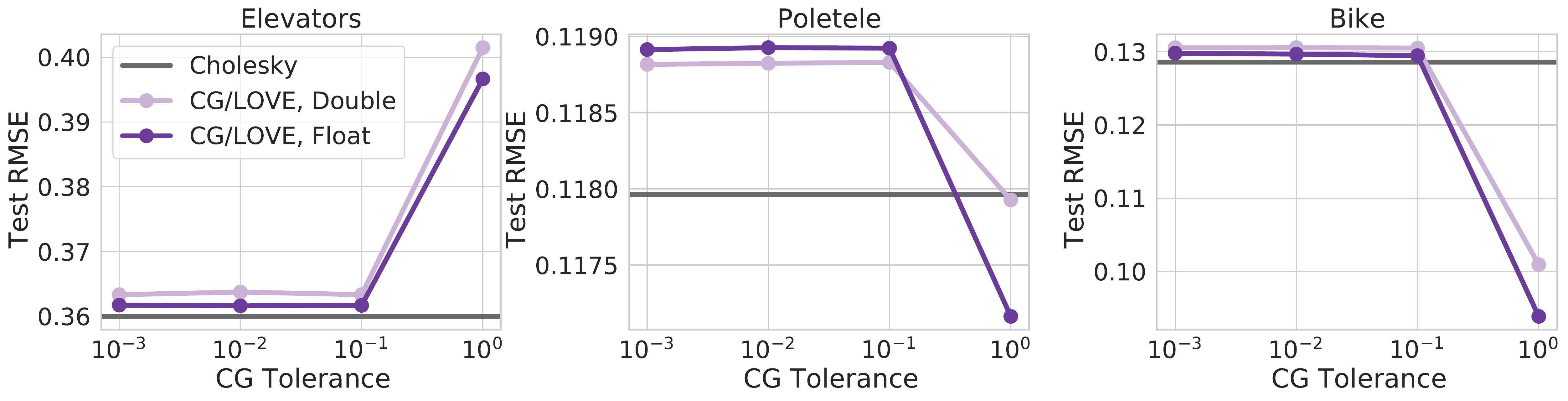}
    \caption{Test set root mean square error (RMSE, standardized) on elevators, bike, and poletele as a function of the CG tolerance for the default rank $15$ preconditioner.}
    \label{fig:rmse_defaults}
\end{figure*}

\subsection{Comparing Adam and L-BFGS-B}
\label{sec:adam_vs_lbfgs}

\begin{figure*}[!ht]
    \centering
    \includegraphics[width=.8\linewidth]{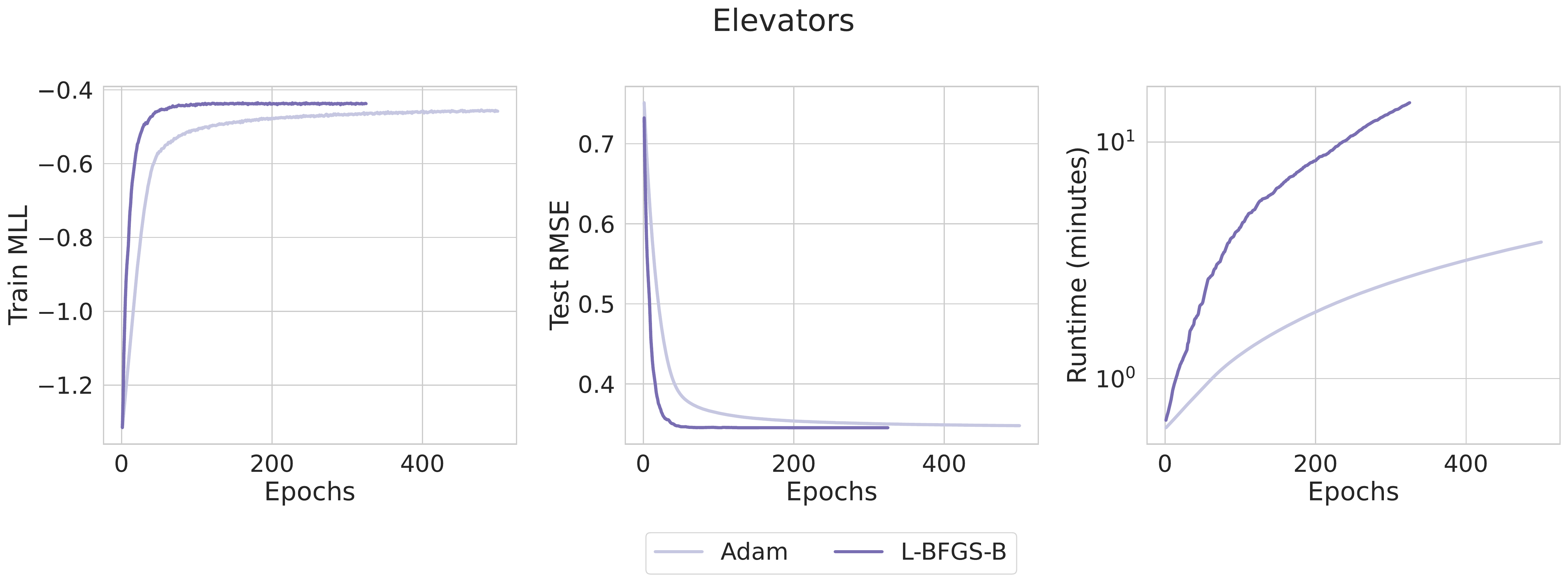}
    \caption{Comparing Adam and L-BFGS-B performance on \texttt{elevators} dataset.}
    \label{fig:elevators_perf}
\end{figure*}

\begin{figure*}[!ht]
    \centering
    \includegraphics[width=.8\linewidth]{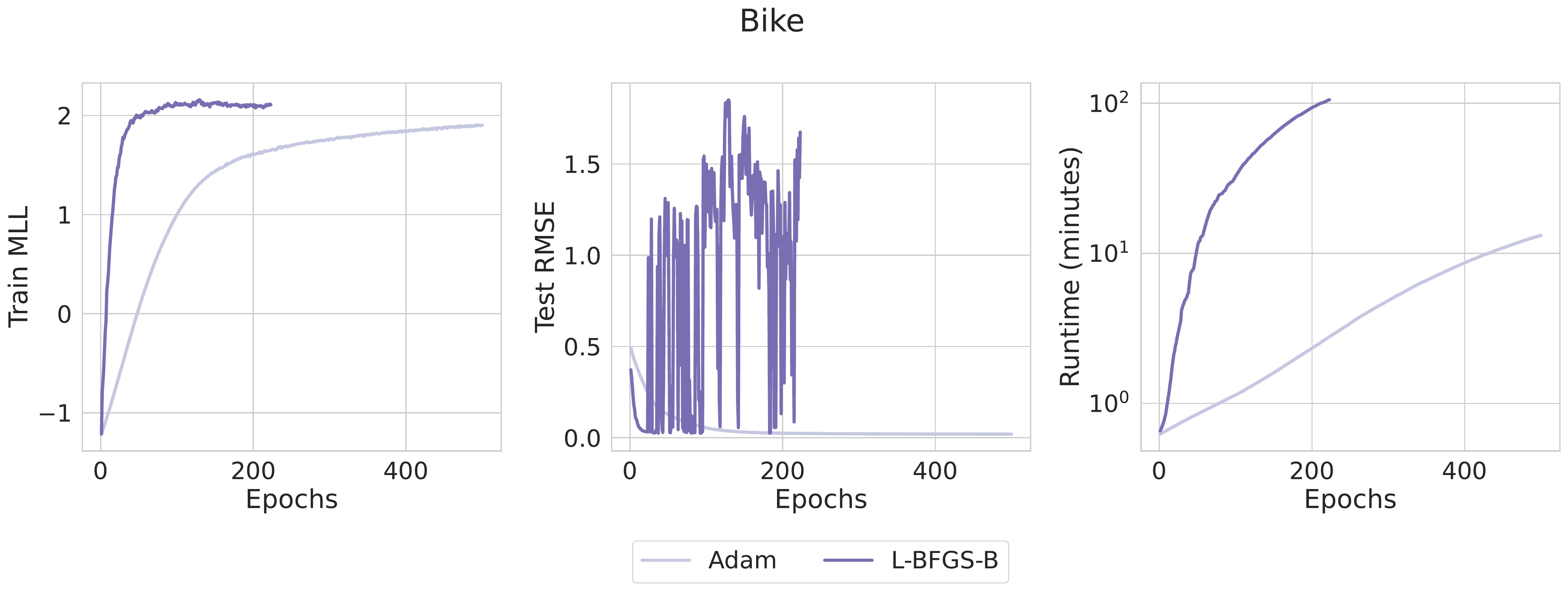}
    \caption{Comparing Adam and L-BFGS-B performance on \texttt{bike} dataset.}
    \label{fig:bike_perf}
\end{figure*}

\begin{figure*}[!ht]
    \centering
    \includegraphics[width=.8\linewidth]{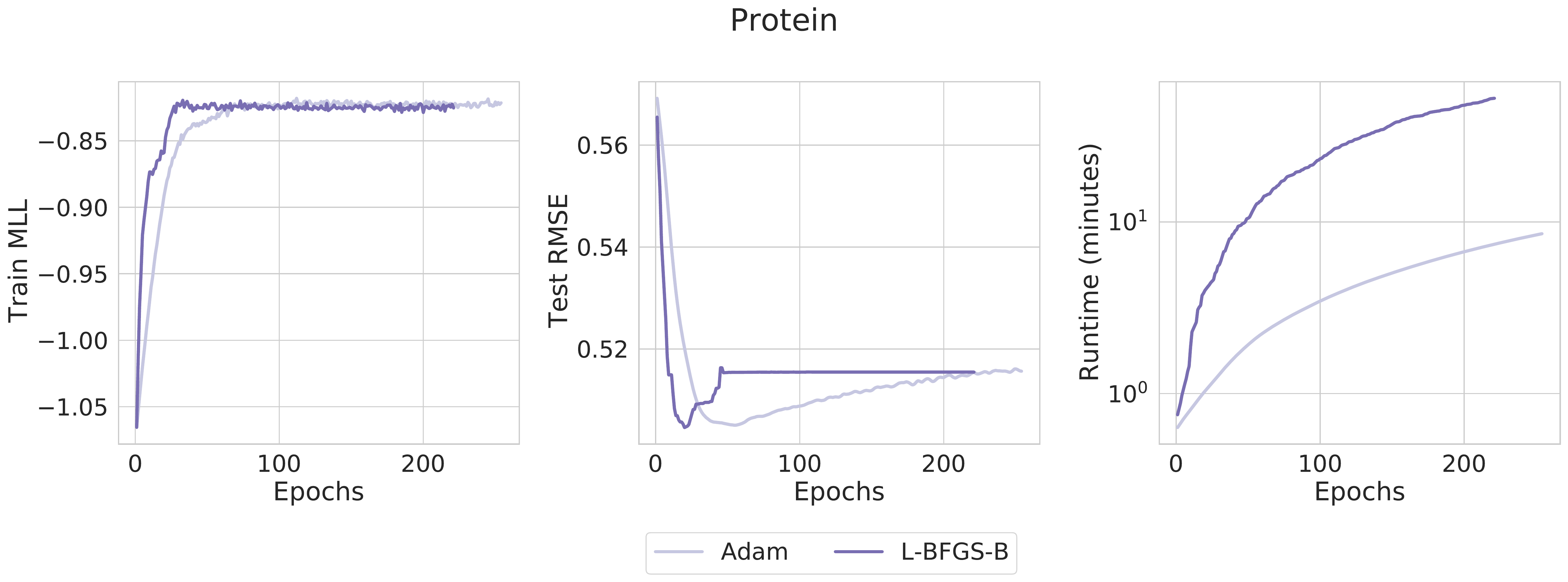}
    \caption{Comparing Adam and L-BFGS-B performance on \texttt{protein} dataset.}
    \label{fig:protein_perf}
\end{figure*}

In \cref{fig:elevators_perf,fig:bike_perf,fig:protein_perf}, we plot the curves for train MLL \cref{eq:gp_mll}, the test RMSE performance, and the runtime in minutes, against the number of epochs.

Overall, we find that while both optimizers achieve the same train MLL and test RMSE, L-BFGS-B converges much faster than Adam. Unfortunately, the gains achieved here are defeated by the L-BFGS-B being much slower. This is a practical challenge to be resolved in future work.
\end{document}